\documentclass[letterpaper, 10 pt, conference]{ieeeconf}  

\IEEEoverridecommandlockouts                              

\usepackage{graphicx} 
\usepackage{amsmath} 
\usepackage{amssymb}  
\usepackage{ntheorem}
\usepackage[ruled,linesnumbered]{algorithm2e}
\usepackage{booktabs}
\usepackage{subcaption}
\usepackage{float}
\usepackage{placeins}
\usepackage{xcolor}
\usepackage{balance}
\usepackage{cite}
\usepackage{multirow}
\usepackage{adjustbox}

\newtheorem{myDef}{Definition}

\newtheorem{myProb}{Problem}
\newtheorem{myProp}{Proposition}

\title{\LARGE \bf 
Online Synthesis of Control Barrier Functions with Local Occupancy Grid Maps for Safe Navigation in Unknown Environments}

\author{Yuepeng Zhang, Yu Chen, Yuda Li, Shaoyuan Li and Xiang Yin
\thanks{This work was supported by the National Natural Science Foundation of China (62061136004, 62173226, 61803259). 
 }
\thanks{Y. Zhang, Y. Chen, Y. Li, S. Li and X. Yin 
are with the Department of Automation, Shanghai Jiao Tong University, and the Key Laboratory of System Control and Information Processing, the Ministry of Education of China, Shanghai 200240, China. 
{\tt  E-mail: \{singmal,yuchen26,yuda.li,syli,yinxiang\}@sjtu.edu.cn}.}
}

\begin{document}

\maketitle
\thispagestyle{empty}
\pagestyle{empty}
\setlength{\abovecaptionskip}{0pt}
\setlength{\belowcaptionskip}{3pt}
\setlength{\textfloatsep}{0pt}

\begin{abstract}
Control Barrier Functions (CBFs) have emerged as an effective and non-invasive safety filter for ensuring the safety of autonomous systems in dynamic environments with formal guarantees. However, most existing works on CBF synthesis focus on fully known settings. Synthesizing CBFs online based on perception data in unknown environments poses particular challenges. Specifically, this requires the construction of CBFs from high-dimensional data efficiently in real time.  
This paper proposes a new approach for online synthesis of CBFs directly from local Occupancy Grid Maps (OGMs). Inspired by steady-state thermal fields, we show that the smoothness requirement of CBFs corresponds to the solution of the steady-state heat conduction equation with suitably chosen boundary conditions. By leveraging the sparsity of the coefficient matrix in Laplace's equation, our approach allows for efficient computation of safety values for each grid cell in the map.  
Simulation and real-world experiments demonstrate the effectiveness of our approach. Specifically, the results show that our CBFs can be synthesized in an average of milliseconds on a \(200 \times 200\) grid map, highlighting its real-time applicability.
\end{abstract}
\section{Introduction}
Safety is one of the fundamental concerns in autonomous systems, such as mobile robots. 
For example, in search-and-rescue scenarios, robots must navigate through a priori unknown environments while avoiding collisions with obstacles. Similarly, in human-robot collaboration, robots need to interact safely with humans, who may be moving dynamically. 
The problem of safe robot navigation has a long-standing research history and has led to the development of many successful techniques, such as artificial potential fields \cite{khatib1986real}, the dynamic window approach \cite{fox1997dynamic}, and deep learning-based methods \cite{nguyen2024uncertainty, jacquet2024n}.  
Although many of these techniques have proven practically successful, they often lack formal safety guarantees. 
Such guarantees are essential for \emph{safety-critical systems}, as they ensure safety under all conditions, regardless of the behavior of the environment or the actions of dynamic obstacles.

Recently, control barrier functions (CBFs) have emerged as a formal approach for ensuring and certifying safety in safety-critical systems \cite{ames2017control,ames_control_2019,xiao2023safe}. 
The core idea of this approach is to define a safe set using a continuously differentiable function, known as the CBF, such that forward invariance of the safe set can be rigorously guaranteed. 
Over the past years, CBF techniques have progressed significantly, with the development of many variants, such as high-order CBFs \cite{xiao2022high,tan2022high}, adaptive CBFs \cite{taylor2020adaptive,xiao2022adaptive}, robust CBFs \cite{buch2021robust}, and input-to-state safe CBFs \cite{kolathaya2018input,alan2021safe}. Additionally, CBFs have been successfully applied to a wide range of autonomous systems in safety-critical scenarios, including unmanned aerial vehicles \cite{tayal_control_2024}, legged robots \cite{grandia_multi-layered_2021}, wheeled robots \cite{black2024cbfkit}, and multi-robot systems \cite{wang2017safety}.

Most of the aforementioned works on CBFs focus on fully known or static settings, where the positions of obstacles are fixed and known a priori. However, in many real-world applications, such as search-and-rescue missions or human-robot collaboration, obstacles are unpredictable and can change dynamically in real-time. As a result, it becomes necessary to synthesize CBFs on-the-fly based on real-time perception data to adapt to unknown and dynamic environments. 
These online perception-based settings introduce new challenges to the synthesis of CBFs. First, the CBFs needs to be constructed directly from sensor data, which captures dynamic obstacle information. Such perception data are often high-dimensional or multi-modal, requiring efficient processing to extract relevant safety-critical information. Additionally, due to the online nature, the synthesis of CBFs must be performed in a lightweight and  efficient manner to ensure real-time computation within each control period. 

To synthesize CBFs from real-time sensor data, several approaches have been proposed recently, including the use of support vector machines \cite{srinivasan_synthesis_2020}, sparse Bayesian learning \cite{mizuta_safe_2022}, and Gaussian process regression \cite{keyumarsi_lidar-based_2024}. However, these methods generally exhibit high computational complexity and often fail to meet real-time requirements when processing large volumes of sensor data.   
To address the challenge of computational efficiency, deep neural networks have emerged as a promising alternative for online CBF synthesis. For instance, in \cite{abdi_safe_2023}, the authors utilize conditional Generative Adversarial Networks (cGANs) to map front-view RGB-D images directly to CBFs. Similarly, in \cite{long_learning_2021}, an online incremental training method is proposed, which employs replay memory and a deep neural network to approximate Signed Distance Functions (SDFs) from LiDAR data. While these neural network-based methods demonstrate efficiency in real-time computation, they often require extensive pre-training on large datasets and may struggle to adapt quickly to changes in new or evolving environments.


To better address the real-time computation requirements for online CBF synthesis,   the authors in \cite{jian_dynamic_2022,zhang_online_2024} propose a method where point-cloud data is first clustered, and each obstacle is represented by a minimal bounding ellipse (MBE). 
A separate CBF is then constructed analytically based on the MBE representation of the obstacle. 
However, this MBE-based approximation can be overly conservative, particularly when obstacles are non-convex.
Recently, in \cite{raja_safe_2024}, the problem of online synthesis of CBFs from occupancy grid maps  has also been studied. 
The approach involves designing CBFs based on SDFs. However, SDFs are not differentiable everywhere, which poses a challenge for ensuring the smoothness of CBFs derived from them. As a result, careful consideration is required, often involving advanced interpolation and smoothing techniques to address non-differentiable regions and ensure the resulting CBFs are sufficiently smooth for safe control synthesis.  

In this paper, we address the challenge of synthesizing  CBFs  in real-time based on perception data in the form of a local occupancy grid map (OGM). Inspired by steady-state thermal fields in physics, we propose a novel approach for synthesizing CBFs from OGMs that achieves a better tradeoff between conservatism, adaptability, and computational efficiency.  
The core idea of our approach is to characterize the requirements of CBFs using the Laplace's heat conduction equation in steady-state thermal fields. Two boundary conditions are employed to represent obstacles and absolute safe regions. 
Additionally, we provide an efficient numerical method to synthesize the CBF without explicitly solving the partial differential equation by leveraging the sparsity of the coefficient matrix in the Laplace's equation.   
Compared to the MBE-based approach, our method requires only a single CBF constraint to ensure safety during robot navigation, regardless of the number or shape of obstacles in the environment, which makes our approach less conservative. 
Compared with the SDF-based approach, our method avoids the need for complex interpolation and smoothing techniques, ensuring smoother and more computationally efficient CBFs.   
Simulations and real-world experiments are conducted to validate the effectiveness and real-time efficiency of our approach. The experimental results demonstrate that our method can synthesize  CBFs  in milliseconds for 200 \(\times\) 200 occupancy grid maps (OGMs). This highlights the practicality and scalability of our approach for real-time safe control in dynamic  environments.

The organization of this paper is as follows: First, Section \ref{section:preliminaries} reviews the foundational theories of CBFs. Next, Section \ref{section:problem} outlines the problem statement. In Section \ref{section:SSTF-CBF}, a novel CBF inspired by steady-state thermal fields is proposed. Following that, Section \ref{section:experiment} evaluates the performance of this method through simulations and real-world experiments. Finally, Section \ref{section:conclusion} concludes the paper and explores potential directions for future research.

\section{Preliminaries}
\label{section:preliminaries}
In this work, we consider a mobile robot modeled by an affine nonlinear control system:
\begin{equation}\label{eq:dynamics}
\dot{\mathbf{x}} = \mathbf{f}(\mathbf{x})+\mathbf{g}(\mathbf{x})\mathbf{u}, 
\end{equation}
where 
$\mathbf{x} \in \mathcal{X} \subseteq \mathbb{R}^n$ is the system state, 
$\mathbf{u} \in \mathcal{U} \subseteq \mathbb{R}^m$ is the control input, 
$\mathbf{f}: \mathbb{R}^n \to \mathbb{R}^n$ and $\mathbf{g}: \mathbb{R}^n \to \mathbb{R}^{n \times m}$ are locally Lipschitz continuous functions representing the dynamic of the systems. 
We assume that the state space is partitioned as
\begin{equation}
\mathcal{X}=\mathcal{X}_{\text{obs}}\dot{\cup}\mathcal{X}_{\text{free}}, 
\end{equation}
where $\mathcal{X}_{\text{obs}}$ is a closed subset representing the obstacle region the robot must avoid, and  
$\mathcal{X}_{\text{free}}$ is an open subset representing the safe region where the robot is free to navigate.

To ensure the safety of the system,  
our objective is to identify a \emph{safe set} $\mathcal{C} \subseteq \mathcal{X}_{\text{free}}$ such that the robot's motion remains within $\mathcal{C}$, thereby guaranteeing safety. To achieve this, the set $\mathcal{C}$ must be forward invariant.

\begin{myDef}[Forward Invariant Sets \cite{ames_control_2019}]\upshape
A subset $\mathcal{C}\subseteq \mathcal{X}_{\text{free}}$ is said to be \emph{forward invariant} if, 
for any initial state $\mathbf{x}(t_0) \in \mathcal{C}$, there exists a control sequence $\mathbf{u}(t)\in \mathcal{U}$ such that  $\mathbf{x}(t) \in \mathcal{C}$ for all future times $t \geq t_0$. 
\end{myDef}

The safe set $\mathcal{C}$ is typically characterized as the 0-superlevel set of a continuously differentiable safety function $h(\mathbf{x}): \mathbb{R}^n \to \mathbb{R}$. Specifically, the safe set is defined as follows:
\begin{align}
\label{eq:safe set}
\mathcal{C} &= \{ \mathbf{x} \in  \mathbb{R}^n \mid h(\mathbf{x}) \geq 0 \},\\
\label{eq:safe boundary}
\partial \mathcal{C} &= \{ \mathbf{x} \in \mathbb{R}^n \mid h(\mathbf{x}) = 0 \},\\
\label{eq:unsafe set}
\text{Int}(\mathcal{C}) &= \{ \mathbf{x} \in \mathbb{R}^n \mid h(\mathbf{x}) > 0 \}.
\end{align}
Specifically,  
when $h(\mathbf{x})$ is a \emph{control barrier function}, the forward invariance of $\mathcal{C}$ can be ensured by synthesizing safe control inputs at each time instant.

\begin{myDef}[Control Barrier Functions \cite{li2023robust}]\label{def:cbf}\upshape
Consider a system defined in \eqref{eq:dynamics}.  
A continuously differentiable function $h: \mathbb{R}^n \to \mathbb{R}$ is said to be a \emph{control barrier function} if, for any state $\mathbf{x} \in \mathcal{C}$, we have
\begin{align}
\label{eq:constraint}
\sup_{\mathbf{u}\in \mathcal{U}}[L_{\mathbf{f}}h(\mathbf{x})+L_{\mathbf{g}}h(\mathbf{x})\mathbf{u}] \geq -\alpha(h(\mathbf{x})),
\end{align}
where $\alpha:\mathbb{R}\rightarrow \mathbb{R}$ is an extended class $\mathcal{K}_{\infty}$ function, i.e., strictly increasing with $\alpha(0)=0$.
\end{myDef}
The set of control inputs that renders $\mathcal{C}$ safe is given by 
\begin{align}
K_{cbf}(\mathbf{x}) = \{\mathbf{u} \in \mathcal{U}: L_{\mathbf{f}}h(\mathbf{x})\!+\!L_{\mathbf{g}}h(\mathbf{x})\mathbf{u} \geq\! - \alpha(h(\mathbf{x}))\}.
\end{align}
That is, when $h(\mathbf{x})$ is a control barrier function, any control input $u \in K_{\text{cbf}}(\mathbf{x})$ ensures the safety of the system.

\section{Problem Statement}\label{section:problem}
When the environment is static and known a priori, e.g., the robot is aware of the positions of all obstacles that remain unchanged during operation, a control barrier function can be synthesized offline. 
Then this pre-computed static control barrier function can then be utilized online to generate safe control inputs.
However, in many scenarios, control barrier functions must be synthesized on-the-fly based on the  real-time perception data. 
Examples of such scenarios include:  
\begin{itemize}  
    \item
    \emph{Exploration in Unknown Terrains}: The environment is initially unmapped, and the robot must simultaneously build a map of the surroundings and compute safe control inputs on-the-fly.  
    \item 
    \emph{Human-Robot Collaborations}: The obstacles, such as humans, are dynamic rather than fixed. As a result, the robot must continuously adapt in real-time to the movements of human operators to ensure safe interaction.  
\end{itemize}  
In these cases, the ability to synthesize control barrier functions on-the-fly is essential for ensuring safe and adaptive operation.
 
In this work, we consider the setting where the robot relies on real-time perception data in the form of a \emph{local occupancy grid map} (OGM). To be more specific, we consider a mobile robot navigating within an 
$x$-$y$ plane, with its kinematic model described by \eqref{eq:dynamics}. The state of the robot is defined by 
$\mathbf{x} = (x, y, \theta)$, where $\mathbf{p}=(x, y)$ represents its coordinates on the plane and $\theta$ denotes its orientation angle.
We assume that the robot can continuously and precisely determine its state, for instance, by utilizing odometry or GPS signals. Additionally, the robot is equipped with local sensing capabilities, enabling it to construct a local occupancy grid map centered around its position at each time step. Such local perception can be achieved, for example, through Bird’s Eye View (BEV) images captured by cameras, combined with semantic segmentation techniques, or via LiDAR systems. 
Specifically, the OGM obtained by the robot at time instant $t$ is represented as a binary matrix:
\begin{equation}
M_t \in \{0,1\}^{H \times W},
\end{equation} 
where $M_t(i,j) = 1 $ indicates that the grid cell  $(i,j) $ in the  $ H \times W $ map is occupied by an obstacle, and  $0 $ indicates that the cell is free of obstacles.
This scenario of local OGM is illustrated in Figure \ref{fig: problem}.  
\begin{figure}[t]
  \centering
  \includegraphics[scale=0.25]{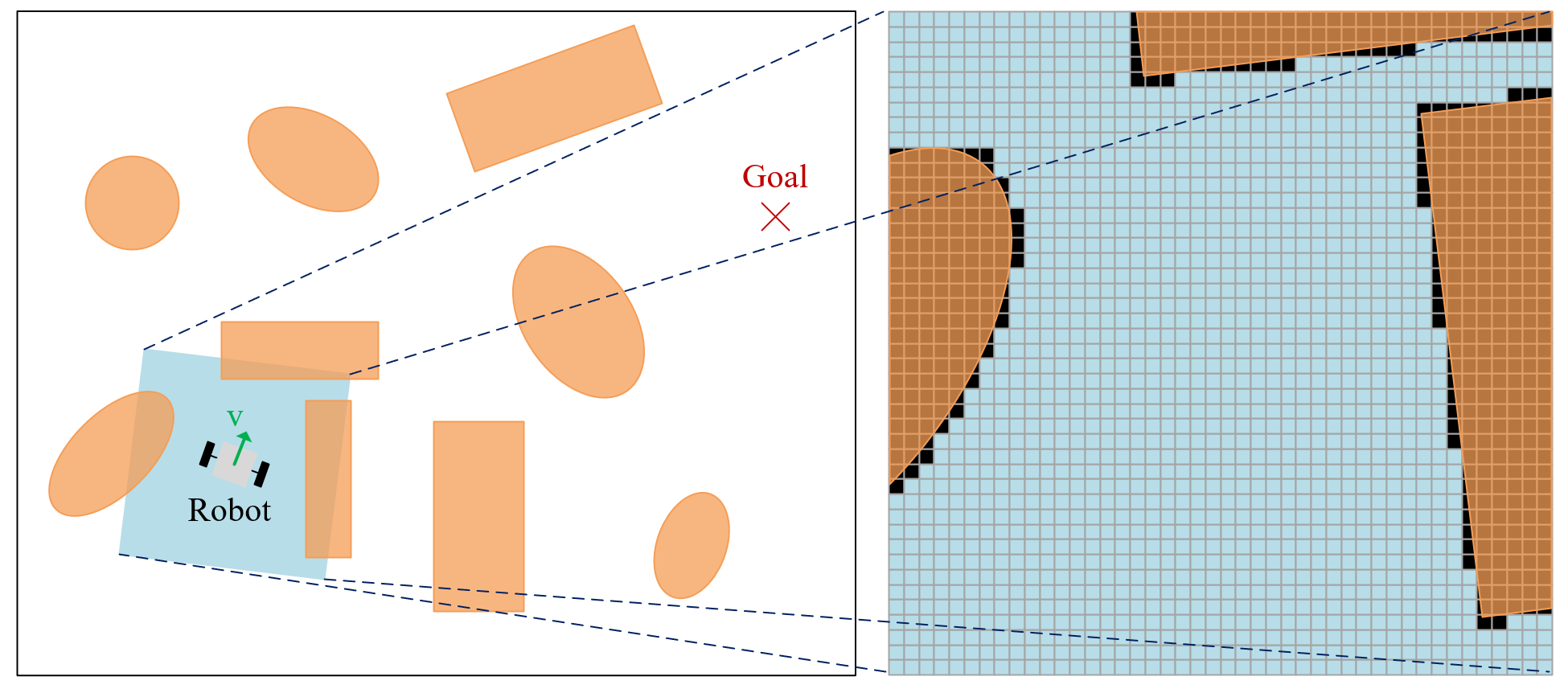}
  \caption{(\textbf{Left}) The robot's navigation environment, where orange regions denote obstacles and the blue region represents the robot's local sensing range. (\textbf{Right}) The real-time occupancy grid map generated during navigation, where black grids indicate obstacle-occupied areas while blue grids correspond to free space.}
  \label{fig: problem}
\end{figure}

In robot navigation tasks, a high-level controller, such as proportional control, model predictive control (MPC) or manual control through teleoperation, are used to generates nominal control inputs $u_{\text{nom}}(t)$ to guide the robot towards a specified goal $\mathbf{p}_{goal}$. 
However, the robot's workspace $\mathcal{W} \subseteq \mathbb{R}^2$ contains static or dynamic obstacles of unknown quantity and shape, denoted as $O_1, \dots, O_N$, which could result in potential collisions if only $u_{\text{nom}}(t)$ is applied. 
To prevent collisions and ensure safe navigation, the robot's position $\mathbf{p}=(x, y)$ must avoid entering obstacle regions $O =\bigcup_{i=1}^{N}O_i$. Since the robot's position $\mathbf{p}=(x, y)$ is the primary factor for ensuring safety in a 2D plane, we adopt the widely used 2D single integrator dynamics for analysis. This choice is motivated by the model's simplicity and its ability to capture the essential dynamics of position-based safety constraints. For higher-dimensional system models, extensive research has shown that constructing CBFs for simlified models is an effective approach. These CBFs can then be optimized to ensure the safety of more complex full-order models (FOMs) \cite{cohen2024safety, molnar2021model, molnar2023safety}. For example, in \cite{cohen2024safety2}, the single integrator dynamics, as a reduced-order model (ROM), is utilized to design CBF constraints. The resulting safe control inputs are subsequently mapped back to the FOM, transferring safety guarantees to a 3D hopping robot. In our experiments, we also successfully mapped the safe control inputs from the single integrator dynamics back to a unicycle model, achieving safe obstacle avoidance during navigation. 

To ensure safety, a key requirement is to find a continuously differentiable CBF $h(\mathbf{p}): \mathbb{R}^2 \to \mathbb{R}$ that is synthesized by $M_t$. This CBF partitions the $x-y$ plane into a safe set $\mathcal{C}=\{\mathbf{p}\in\mathcal W \mid  h(\mathbf{p})\geq 0\}$ and an unsafe set $\mathcal{C}_u=\mathcal W \setminus \mathcal{C}$, where the obstacle region $O$ belongs to the unsafe set, i.e., $O \subseteq \mathcal{C}_u$. When the robot starts from a position within the safe set ($\mathbf{p}(0) \in \mathcal{C}$), imposing constraints on the control inputs according to \eqref{eq:constraint} ensures that the robot avoids collisions with obstacles throughout its motion. Furthermore, as the robot moves and the surrounding environment may dynamically change, the CBF must be capable of updating in real time or being synthesized online to maintain safety. The problem addressed in this paper can be described as follows.

\begin{myProb}
Consider a robot performing navigation tasks in an environment with unknown obstacles, where its kinematic model is described by \eqref{eq:dynamics}. 
Design a CBF \( h(\mathbf{p}) \) that can be updated in real time or synthesized online based on the OGM \( M_t \) acquired by the robot.
\end{myProb}

\section{Steady-State Thermal Field-Inspired CBFs}
\label{section:SSTF-CBF}
A common approach to addressing the above online synthesis problem involves encapsulating each obstacle with circles or ellipses, and employing distance functions as CBFs. However, this method requires designing a separate CBF constraint for each obstacle, which can be overly conservative and computationally intensive. To overcome these limitations, we draw inspiration from the physical model of steady-state thermal fields. By solving the associated partial differential equations, we propose a method to synthesize a unified CBF that works for any number or shape of obstacles, ensuring that only a single CBF constraint needs to be satisfied to guarantee safety. 

\subsection{Safety Function Synthesis}
\label{section:CBF-Synthesis}
Determining a steady-state thermal field requires defining clear boundary conditions. These boundary conditions are necessary for solving the safety function and ensure $h(\mathbf{p})$ exhibits appropriate properties throughout the workspace.  First, we define a safety region $S \subseteq \mathcal{W}$ by
\begin{equation}
    S=\{ p \in \mathcal{W} \mid \inf_{q \in O} \| p-q \| \geq \delta \},
\end{equation}
where $\| \cdot \|$ denotes the Euclidean distance and $\delta$ is a user-defined safety margin. Next, we model the boundary of the obstacle region $\partial O$ and the boundary of the safety region $\partial S$ as heat sources with fixed temperatures of $-a$ and $b$, respectively, where $a,b>0$. 
This ensures that after reaching steady-state, the temperature (or safety value) is negative near obstacles and increases as we move away from them towards the safety region. 
We define this gradually changing area as the transition region $\mathcal{T}=\mathcal{W}\setminus (O \cup S)$. Additionally, the boundary $\partial \mathcal W$ of the workspace is treated similarly to $\partial S$, as we want the safety function to remain valid across the entire workspace. Formally, we define: 
\begin{align} \label{eq:CBFboundary}
B(\mathbf{p}) = 
\begin{cases} 
-a, & \quad \mathbf{p} \in \partial O \\
\ \ b, & \quad \mathbf{p} \in \partial S \cup \partial \mathcal W
\end{cases},
\end{align}
where $B(\mathbf{p})$ acts as a Dirichlet boundary condition that imposes fixed safety values at the boundaries.

This setup induces a temperature gradient, driving heat conduction from higher to lower temperatures until steady-state is reached. The steady-state temperature distribution satisfies Laplace's equation \cite{hahn2012heat}:
\begin{align}
\label{eq:laplace_equation}
\nabla^2 h(\mathbf{p}) = \frac{\partial^2 h(\mathbf{p})}{\partial x^2} + \frac{\partial^2 h(\mathbf{p})}{\partial y^2} =0.
\end{align}
In this context, $h(\mathbf{p})$ is the solution to \eqref{eq:laplace_equation} over $\mathcal{W}$ with boundary condition \eqref{eq:CBFboundary}. We term this the steady-state thermal field-inspired control barrier function (SSTF-CBF). Figure \ref{fig:SSTF-CBF}, as an example, shows the distribution of obstacles and the corresponding color map of the generated $h(\mathbf{p})$. This method offers several advantages, including a clear physical interpretation, the ability to adapt to complex boundaries, and the guarantee that the safety function $h(\mathbf{p})$ is infinitely differentiable where the Laplace's equation holds \cite{axler2001basic}.
\begin{figure}[t]
  \centering
   \includegraphics[scale=0.28]{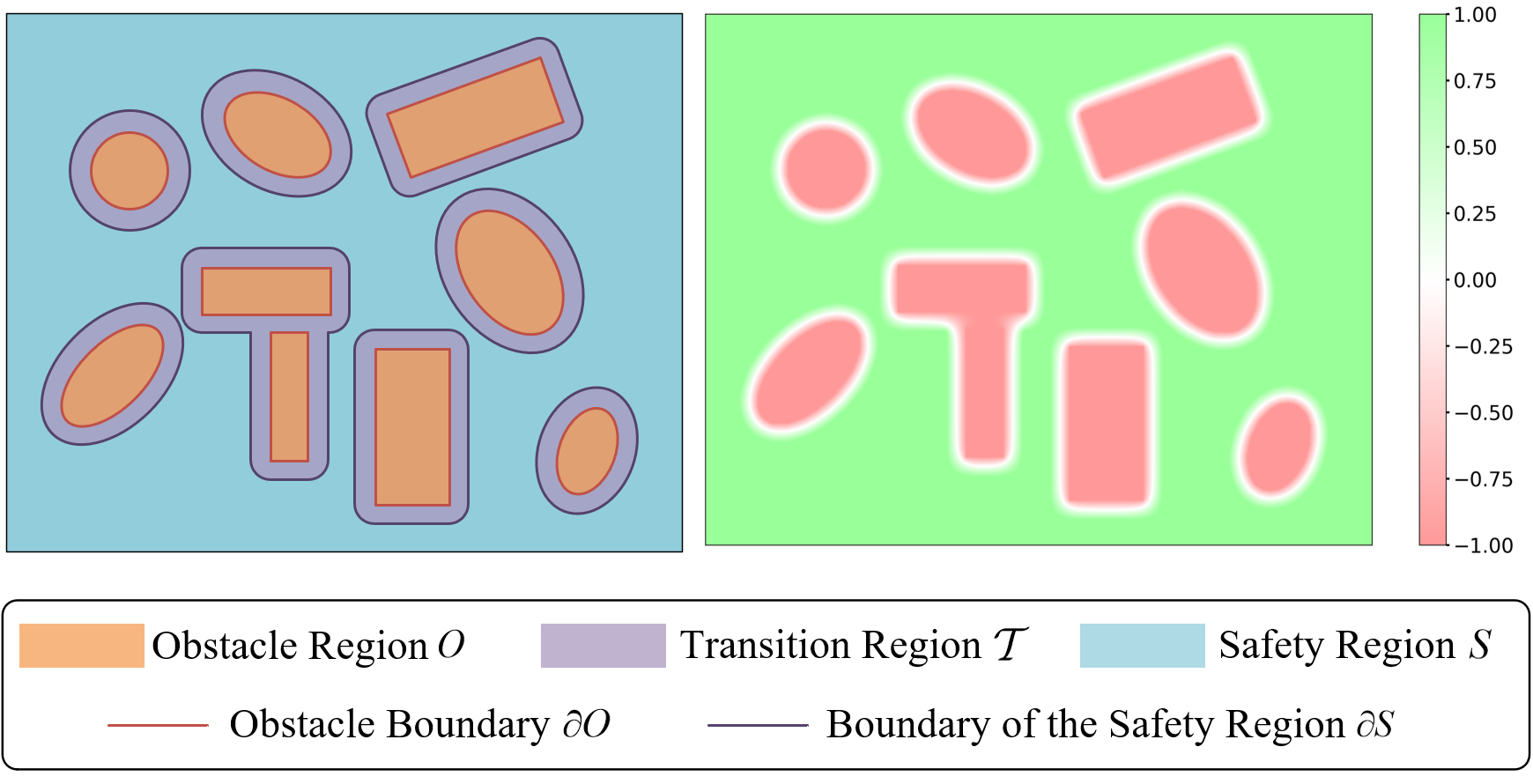}
  \caption{Obstacle distribution and color map of the corresponding steady-state thermal field safety function $h(\mathbf{p})$, computed with $a=b=1$.}
  \label{fig:SSTF-CBF}
\end{figure}
\label{section:CBF-Control}

Next, we will demonstrate how the sparsity of the Laplace equation's sparse matrix enables the efficient synthesis of $h(\mathbf{p})$ on the OGMs. Suppose the robot generates a $H \times W$ grid map $M$ to represent obstacle occupancy, we first identify the boundaries of the obstacle region $\partial O$ (from $M$) and the safety region $\partial S$ (via dilation of $M$). These boundaries are treated as heat sources with fixed temperatures:$-a$ for $\partial O$ and $b$ for $\partial S$. In steady state, temperatures in regions $O$ and $S$ stabilize at $-a$ and $b$, respectively, due to their exclusive connections to single-valued heat sources. Therefore, we only need to compute the safety values within the transition region $\mathcal{T}=M\setminus (O \cup S)$. The safety values for each grid cell within $\mathcal{T}$ satisfy the second-order difference form of Laplace's equation \eqref{eq:laplace_equation}:
\begin{align}
\label{eq:differ_form}
h_{i,j} = \frac{h_{i+1,j}+h_{i-1,j}+h_{i,j-1}+h_{i,j+1}}{4},
\end{align}
where $h_{ij}$ represents the safety value of the grid cell at the $i$-th row and $j$-th column in the map $M$. Assuming there are $N$ grid cells within the region $\mathcal T$, each requiring a calculation of its safety value, we can formulate $N$ Laplace's equations. Arranging the unknown safety values into a vector yeilds $\textbf{h} = [h_{1j_{11}}, h_{1j_{12}}, ..., h_{1j_{1n_1}},h_{2j_{21}},...,h_{2j_{2n_2}}, ...,h_{Hj_{Hn_{H}}}]^\top$,
where $n_i$ denotes the number of grid cells in the $i$-th row of map $M$ that are within the region $\mathcal T$, and $j_{ab}$ represents the column index of the 
$b$-th grid cell in the $a$-th row that lies in $\mathcal T$. Notably, $n_1+n_2+...+n_H=N$. This allows us to express the system of N linear equations in matrix form:
\begin{align}
A\textbf{h} = \textbf{b},
\label{eq:Ax=b}
\end{align}
where $A \in \mathbb{R}^{N \times N}$, $\text{diag}(A) = 4 \cdot \mathbf{1}^N$, and $\mathbf{h}, \mathbf{b} \in \mathbb{R}^N$. 
\begin{figure}[t]
  \centering
  \includegraphics[scale=0.5]{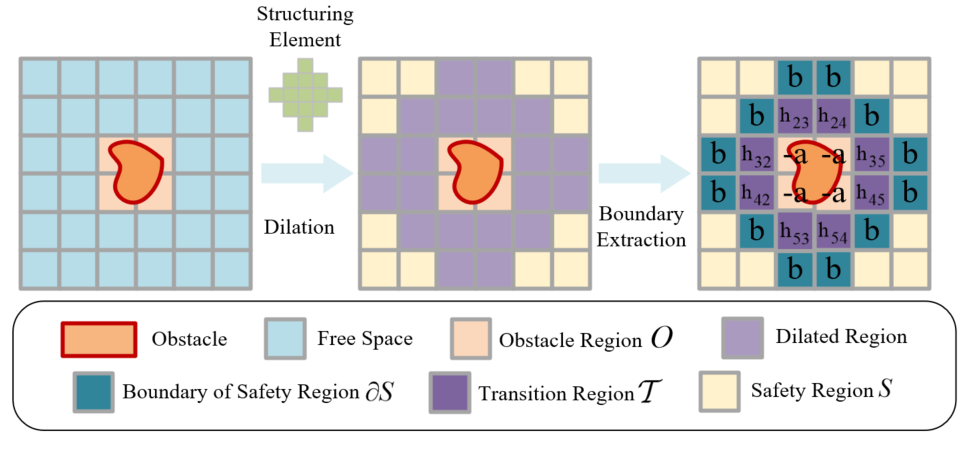}
  \caption{A simple example of the domain definition for our SSTF-CBF is provided, clearly displaying the boundary conditions and the safety values that need to be calculated in transition domain $\mathcal{T}$. In this example, region $O$ and $\partial O$ are identical.} 
  \label{fig:example}
\end{figure}
A simple example can be found in Figure \ref{fig:example}, where the Laplace's equations can be explicitly written as:
\begin{align}
\setlength{\arraycolsep}{0.8pt}
\begin{bmatrix}
    4 & -1 & 0 & 0 & 0 & 0 & 0 & 0 \\ 
    -1 & 4 & 0 & 0 & 0 & 0 & 0 & 0 \\ 
    0 & 0 & 4 & 0 & -1 & 0 & 0 & 0 \\ 
    0 & 0 & 0 & 4 & 0 & -1 & 0 & 0 \\ 
    0 & 0 & -1 & 0 & 4 & 0 & 0 & 0 \\ 
    0 & 0 & 0 & -1 & 0 & 4 & 0 & 0 \\ 
    0 & 0 & 0 & 0 & 0 & 0 & 4 & -1 \\ 
    0 & 0 & 0 & 0 & 0 & 0 & -1 & 4
\end{bmatrix}
\begin{bmatrix}
    h_{23}\\
    h_{24}\\
    h_{32}\\
    h_{35}\\
    h_{42}\\
    h_{45}\\
    h_{53}\\
    h_{54}
\end{bmatrix}=
\begin{bmatrix}
    2b-a\\
    2b-a\\
    2b-a\\
    2b-a\\
    2b-a\\
    2b-a\\
    2b-a\\
    2b-a
\end{bmatrix}.
\end{align}
The coefficient matrix $A$ is constructed based on adjacency relationships among grid elements, resulting in a symmetric matrix with a diagonal of $4$ and off-diagonal elements of $-1$ or $0$. As the size of the grid map $M$ grows, the number of unknown safety values $N$ increases, potentially leading to a very high-dimensional matrix $A$. However, from \eqref{eq:differ_form}, each Laplace equation involves at most five neighboring grid values, making $A$ sparse with no more than five non-zero elements per row. This sparsity allows for efficient solving using iterative methods, such as GMRES \cite{saad1986gmres} or BiCGSTAB \cite{van1992bi}, which are well-suited for large, sparse linear systems.

\subsection{SSTF-CBF-Based Safe Control}
As described in Section \ref{section:problem}, we consider the robot's motion governed by single integrator dynamics:
\begin{equation}
 \label{eq:simplified model}
\dot{\mathbf{p}} = \mathbf{f}(\mathbf{p})+\mathbf{g}(\mathbf{p})u, 
\end{equation}
where $u=(v_x, v_y)$ denotes the control input, with $v_x$ and $v_y$ representing the robot's velocity components in the horizontal and vertical directions. Assume that $u_{\text{nom}}(\mathbf{p}):\mathbb{R}^2 \rightarrow \mathbb{R}$ represents an nominal velocity controller without considering collisions. After synthesizing the SSTF-CBF, the following quadratic programming (QP) problem can be formulated to design a CBF-based safety controller:
\begin{subequations}
\label{eq:controller}
\begin{align}
u^*(t) = \mathop{\arg\min}\limits_{u\in \mathbb{R}^2}\|u-u_{\text{nom}}(t)\|^2 \\
\text{s.t.} \quad L_{\mathbf{f}}h(\mathbf{p})+L_{\mathbf{g}}h(\mathbf{p})u \geq -\alpha(h(\mathbf{p})),
\label{eq:controller_b}
\end{align}
\end{subequations}
where $L_{\mathbf{f}}h(\mathbf{p})=\frac{\partial h}{\partial \mathbf{p}}\mathbf{f}(\mathbf{p})$ and $L_{\mathbf{g}}h(\mathbf{p})=\frac{\partial h}{\partial \mathbf{p}}\mathbf{g}(\mathbf{p})$ are the  Lie derivatives. The partial derivatives $\frac{\partial h}{\partial \mathbf{p}}$ are computed using the differences between adjacent grid cells. The overall algorithm for safe navigation in unknown environments is presented in Algorithm 1.

\begin{algorithm}  
    \caption{Safe Navigation using SSTF-CBF}
    \KwIn{sensor data: $z$, state of the system: $\mathbf{p}$}
    \KwOut{control input $u(t)$}
    \While{Robot is running}
    {
        Receive sensor data $z$ and system state $\mathbf{p}$\;

        Compute the nominal input $u_{\text{nom}}(t)$ at $\mathbf{p}$\;
        
        Generate occupancy grid map $M_t$\;

        \If{No occupied grid cell in $M_t$} {
            Execute $u_{\text{nom}}(t)$\;
            Continue\;
        }
        
        Extract $\partial O_t$\;
        
        Perform dilation on $O_t$\;

        Extract $\partial S_t$\;

        Construct Laplace's equation based on $\partial O_t$ and $\partial S_t$\;

        Solve the sparse Laplace's equation with GMRES or BiCGSTAB\;
        
        Run optimization in \eqref{eq:controller} to compute the safe control input $u(t)$\;
        Execute $u(t)$\;
    }
\end{algorithm}

\begin{myProp}\label{prop:is cbf}
Given a control system (\ref{eq:simplified model}) and the safe set defined as in (\ref{eq:safe set})-(\ref{eq:safe boundary}), the SSTF-CBF $h(\mathbf{p})$  constructed as in section \ref{section:CBF-Synthesis} is a valid control barrier function.
\end{myProp}
\begin{proof}
First, since $h(\mathbf{p})$ satisfies the Laplace's equation (\ref{eq:laplace_equation}), it is continuously differentiable\cite{hahn2012heat}. Second, for every point $\mathbf{p}$ within the safe set $\mathcal{C}$, $h(\mathbf{p})\geq 0$, which implies $-\alpha(h(\mathbf{p}))\leq 0$. Since we consider single integrator dynamics and $u=(v_x, v_y)$ can be made zero, condition (\ref{eq:constraint}) can always be satisfied. Therefore, according to Definition \ref{def:cbf}, $h(\mathbf{p})$ is a valid CBF. 
\end{proof}

In the study of CBFs, it is typically required that $\forall \mathbf{p} \in \partial{\mathcal C}, \frac{\partial h(\mathbf{p})}{\partial \mathbf{p}}\neq 0$ to ensure CBF constraint \eqref{eq:constraint} remains nontrivial. In our proposed SSTF-CBF, there may be positions where $\frac{\partial h(\mathbf{p})}{\partial \mathbf{p}} = 0$ on $\partial{\mathcal C}$. At such points, all control inputs $u$ satisfies the constraint, potentially causing the robot to execute unsafe high-level commands that may drive it into a dangerous region where $h(\mathbf{p})<0$. However, unlike other CBFs that guarantee safety only when $h(\mathbf{p}) \geq 0$, our safety function ensures that the robot remains safe even when $-a<h(\mathbf{p})<0$. Furthermore, since our $h(\mathbf{p})$ is a non-constant harmonic function, all critical points satisfying $\frac{\partial h(\mathbf{p})}{\partial \mathbf{p}} = 0$ are isolated \cite{manfredi1988p}. We will use these properties to prove that even when the gradient of the CBF vanishes at some boundary points, our SSTF-CBF can still guarantee the safety of the robot's motion.

We define the unsafe region $\mathcal{U} =  
\{p \in \mathcal{W} \mid h(p) \leq 0\}$, which is the 0-sublevel set of $h(\mathbf{p})$. Since $h(\mathbf{p})$ is a continuous function, the set $\mathcal{U}$ is necessarily closed. According to the steady-state thermal field constructed in Section \ref{section:CBF-Synthesis},  $h(\mathbf{p}) = b$ exists for all position $\mathbf{p}$ within the safe region $S$. $h(\mathbf{p})\leq0$ occurs only in the bounded region  $D=\mathcal{W}\setminus S=\{p\in\mathcal{W} \mid  \inf_{q\in O}\|p-q\|< \delta\}$.
Therefore, $\mathcal{U}$ must be a bounded closed set (a compact set). As shown in Figure \ref{fig:proof}, assume that the robot enters a connected component $\mathcal{U}_c$ of the unsafe region $\mathcal{U}$ through a boundary point $p_0$ where $\frac{\partial h(p_0)}{\partial p_0}=0$. 

\begin{figure}[t]
  \centering
  \includegraphics[scale=0.3]{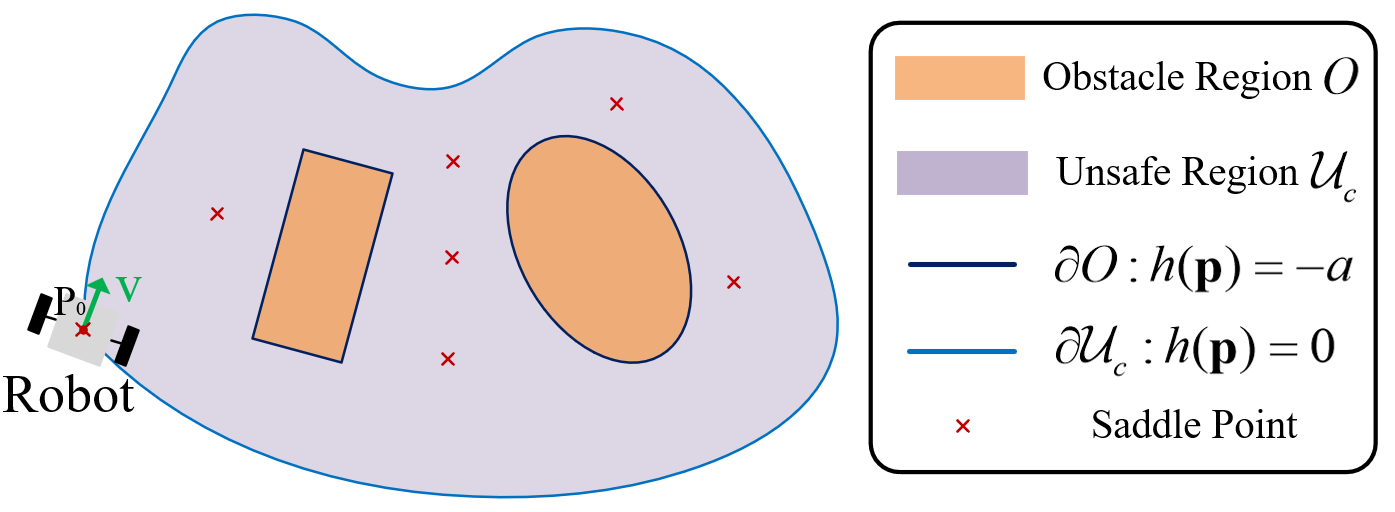}
  \caption{A rough schematic illustrating the robot entering an unsafe region $\mathcal{U}_c$ after passing through $p_0 \in \partial \mathcal{C}$ where the gradient of $h(\mathbf{p})$ vanishes. } 
  \label{fig:proof}
\end{figure}
\begin{myProp}\label{prop:saddle point}
For any compact set $\mathcal{R}\subset \mathbb{R}^2$ contained within the interior of the domain $\mathcal{D}\subset \mathbb{R}^2$ of a non-constant harmonic function $f:\mathcal{D}\to \mathbb{R}$, the number of critical points within $\mathcal{R}$ is finite. 
\end{myProp}

\begin{proof}
For each point \( p \in \mathcal{R} \), there exists a simply connected open neighborhood \( O_p \) of \( p \) such that \( O_p \subset \text{Int}(\mathcal{D}) \). There exists a holomorphic function \( u \) defined on \( O_p \) (viewed as an open subset of \( \mathbb{C} \)) whose real part coincides with \( f \). By the Cauchy Riemann equations, when $\partial f/ \partial x = \partial f/\partial y = 0$, we have $\partial \mathrm{Re}(u)/ \partial x = \partial \mathrm{Re}(u)/\partial y = 0$, which implies $\partial \mathrm{Im}(u)/ \partial x = \partial \mathrm{Im}(u)/\partial y = 0$. This means that the derivative of  u with respect to the complex variable $z=x+iy$ vanishes: $u'(z)=\frac{du}{dz} = 0$. Thus, the critical points of $f$ correspond to the points annihilating \( u' \). Since \( u' \) is holomorphic and not identically zero (because the real part of  \( u \), which is \( f \), is not constant), by the principle of isolated zeros, we conclude that every critical point has an open neighborhood where no other points are critical. By the compactness of \( \mathcal{R} \), we conclude that the critical points in \( \mathcal{R} \) are finite.
\end{proof}

\begin{myProp}\label{prop:safety}
Assume that obstacles $O_1,O_2,…,O_k$ are distributed within $\mathcal{U}_c$. Given a SSTF-CBF as safety function $h(\mathbf{p})$ and under the constraint of (\ref{eq:constraint}), even if the robot enters $\mathcal{U}_c$ through a boundary point $p_0=\mathbf{p}(t_0)$ at time $t_0$ where $h(p_0)=0$ and $\frac{\partial h(p_0)}{\partial p_0}=0$, it will not enter the obstacle region $O_{\mathcal{U}_c}=\bigcup_{i=1}^{k}O_i$ before exiting $\mathcal{U}_c$.
\end{myProp}

\begin{proof}
Since obstacles are present in $\mathcal{U}_c$, the range of $h$ in $\mathcal{U}_c$ is \([-a, 0]\). Define the set $I=\{p\in \mathcal{U}_c|-a+\epsilon\leq h(p)\leq 0\}$, where $\epsilon>0$ is sufficiently small. As $\mathcal{U}_c$
is a connected component of the compact set $\mathcal{U}$, $\mathcal{U}_c$ is also compact. Thus, $I\subseteq \mathcal{U}_c$ is compact. According to Proposition \ref{prop:saddle point}, the number of critical points in $I$ is finite and so is the number of critical points in $\text{Int}(I)$. Since $h(\mathbf{p})$ is continous and non-constant on $I$, by the maximum principle of harmonic functions \cite{axler2001basic}, the extrema of  $h(\mathbf{p})$ in $I$ occurs only on $\partial I$. Therefore, $\forall p \in \text{Int}(I)$, $h(\mathbf{p}) \in (-a+\epsilon, 0)$. Let $v_{max}\in (-a+\epsilon, 0)$ denote the maximum safety value among the critical points in $\text{Int}(I)$. 

We now proceed by contradiction. Suppose the robot enters the obstacle region $h(\mathbf{p}(t_1))=-a$ at time $t_1$ before exiting  $\mathcal{U}_c$. Since $h$ is a is continuous in $\mathbf{p}$ and $\mathbf{p}$ is continous in $t$, $h$ is also continuous in $t$. Thus, before reaching $h(\mathbf{p}(t_1))=-a$, $h$ must attain $v_{max}$ at some earlier time $t_2 \in (t_0, t_1)$, where $t_2$ is the first time at which $h(\mathbf{p}(t))= v_{max}$ after $t_0$. During the motion from $p_0$ to $\mathbf{p}(t_2)$, 
the robot encounter no critical points, and the constraint (\ref{eq:constraint}) remains valid, which implies that
\[\forall t \in (t_0, t_2), \dot{h}(\mathbf{p})=L_{\mathbf{f}}h(\mathbf{p})+L_{\mathbf{g}}h(\mathbf{p})u \geq -\alpha(h(\mathbf{p}))\geq 0.
\]
This contradicts the decrease of $h$ from $h(p_0)=0$ to $h(\mathbf{p}(t_2))=v_{max}$. Hence, the robot can not enter the obstacle region before exiting $\mathcal{U}_c$.
\end{proof}

\section{Experimental Results}
\label{section:experiment}
To valid SSTF-CBF in unknown real-world environments, we conducted experiments using the TurtleBot 3 Burger in both Gazebo simulations and lab scenarios. All computations were performed on a PC with an AMD Ryzen 7 5000H processor and an NVIDIA GeForce RTX 3060 GPU.

\subsection{Nominal Controller and Robot  Model}

We implemented a distance-dependent proportional controller to guide the robot toward the target position$(x_d, y_d)$. This nominal controller generates expected velocities based on the robot's current position $(x_r, y_r)$:
\begin{align}
\left\{
    \begin{array}{l}
    v_{nom, x} = K_p * (x_d - x_r) \\
    v_{nom, y} = K_p * (y_d - y_r)
    \end{array},
\right.
\end{align}
where $K_p = \frac{k}{\sqrt{(x_d - x_r)^2 + (y_d - y_r)^2}}, k \in \mathbb{R}^+$. 
 The velocity command $u = [v_{x},v_{y}]^\top$ is derived by substituting $u_{nom} = [v_{nom, x},v_{nom, y}]^\top$ into \eqref{eq:controller} , ensuring safety while maintaining proximity to the reference.

 \begin{figure}[t]
  \centering
  \includegraphics[scale=0.32]{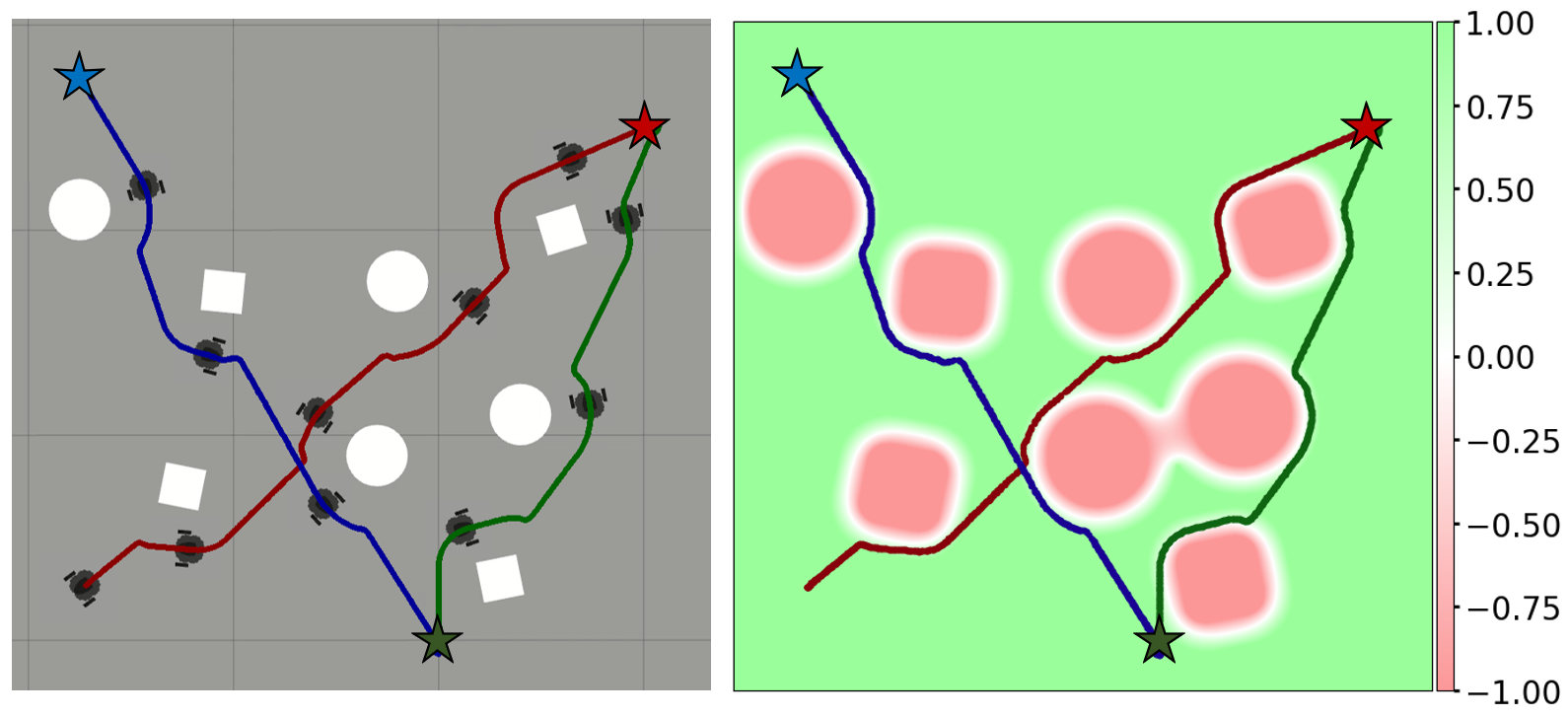}
  \caption{(\textbf{Left}) Obstacle distribution and Turtlebot's trajectory during sequential navigation. (\textbf{Right}) Color map of SSTF-CBF generated by the entire map. The three target points are consistently represented by five-pointed stars in different colors across both figures.}
  \label{fig: sim_traj}
\end{figure}
\begin{figure}[t]
  \centering
  \includegraphics[scale=0.5]{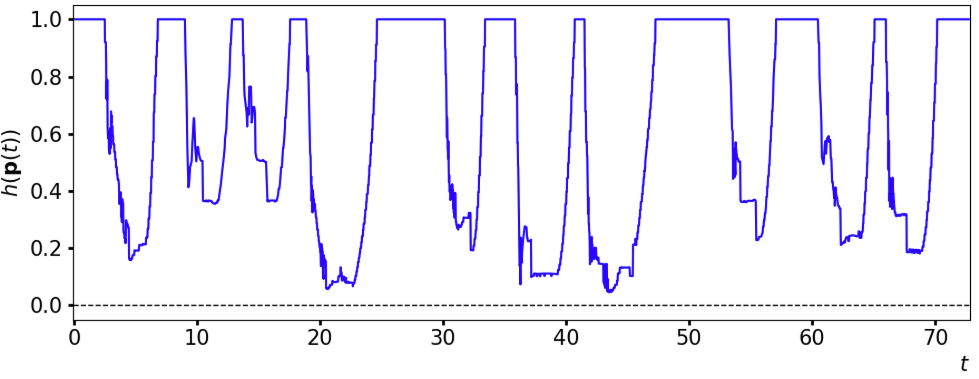}
  \caption{The safety values of the TurtleBot 3 during navigation. These values remain consistently above zero, confirming that the robot operates within the safe region defined by (\ref{eq:safe set}) throughout its motion.}
  \label{fig: sim_h}
\end{figure}

In TurtleBot 3, motion of the robot is described by its center coordinates $(x, y)$ and orientation angle $\theta$, following the unicycle model:
\begin{align}
\begin{bmatrix}
    \dot{x}(t)\\
    \dot{y}(t)\\
    \dot{\theta}(t)
\end{bmatrix}=
\begin{bmatrix}
    \cos(\theta(t)) & 0\\
    \sin(\theta(t)) & 0\\
    0 & 1
\end{bmatrix}
\begin{bmatrix}
    v(t)\\
    \omega(t)
\end{bmatrix},
\end{align}
where $v$ and $\omega$ represent the longitudinal velocity and angular velocity, respectively. 
Control inputs are obtained through a near identity diffeomorphism transformation \cite{wilson2020robotarium}:
\begin{align}
\begin{bmatrix}
    v(t)\\
    \omega(t)
\end{bmatrix}=
\begin{bmatrix}
    \cos(\theta(t)) & \sin(\theta(t))\\
    -\frac{1}{r}\sin(\theta(t)) & \frac{1}{r}\cos(\theta(t))
\end{bmatrix}
\begin{bmatrix}
    v_{x}\\
    v_{y}
\end{bmatrix}.
\end{align}
Here, $r \in \mathbb{R}^+$ denotes the distance from the TurtleBot's center to the wheel axis center.

\subsection{Gazebo Simulations}
First, we design a sequential reaching task in the Gazebo simulator to evaluate the obstacle avoidance capabilities. 
As shown in Figure \ref{fig: sim_traj}, the TurtleBot robot starts at $(1.25\text{ m},1.25\text{ m})$ and sequentially navigates to three goals: $(4 \text{ m}, 3.5 \text{ m}), (3\text{ m}, 1\text{ m})$ and $(1.25\text{ m}, 3.75\text{ m})$. A target is considered reached when the robot is within $0.5$ cm of the destination, triggering the next goal. The $3\text{ m} \times 3\text{ m}$ arena contains four circular obstacles (radius $0.15$ m) and four rectangular obstacles ($0.2 \text{ m} \times 0.2 \text{ m}$).

The robot generates a real-time $200 \times 200$ local occupancy grid map for online SSTF-CBF synthesis, with each grid corresponding to a $1 \text{ cm} \times 1 \text{ cm}$ area. This local map is used for obstacle avoidance during navigation.
Additionally, the robot is aware of its current position. To prevent collisions, the TurtleBot is enclosed within a circular boundary of radius $0.1 \text{ m}$, and its maximum velocity is set to $0.15m/s$.

The TurtleBot's trajectory and the SSTF-CBF color map, which is generated based on the global occupancy grid map for visualization purposes, are also provided in Figure \ref{fig: sim_traj}. The parameters for SSTF-CBF are set to $a = b = 1$ and $\delta = 0.15$ m. To account for the robot's radius, a dilation operation was applied to the OGM before CBF synthesis, causing obstacles in the color map to appear larger than their actual size. For navigation, the function $\alpha(h)=0.15h$ was employed in \eqref{eq:controller_b}. Additionally, the variations in the safety value $h(\mathbf{p})$ were monitored as the robot approached the three target points, as shown in Figure \ref{fig: sim_h}. The results confirm that the robot successfully and safely reaches all target points without collisions. 

\begin{figure*}[t]
    \centering
    \begin{subfigure}{0.32\textwidth}
        \centering
        \includegraphics[width=\textwidth]{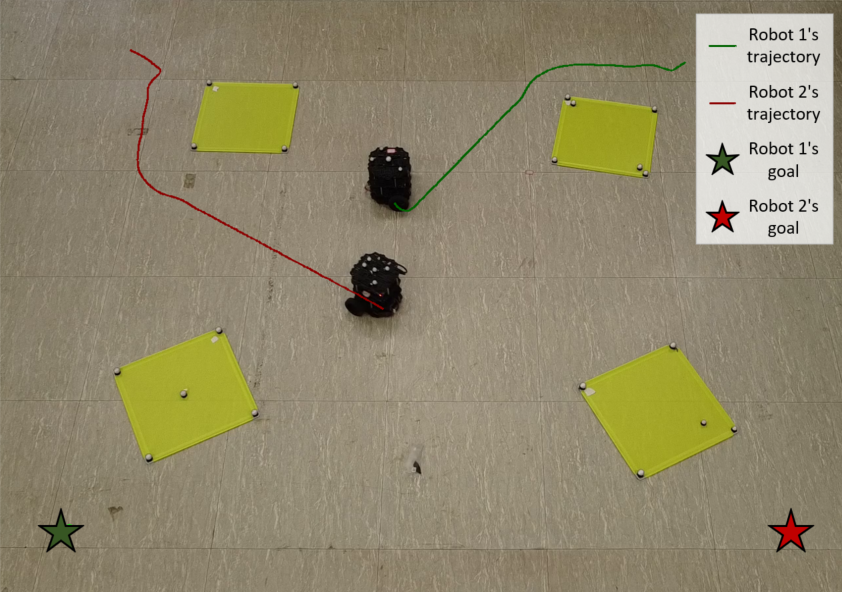}
        \caption{$t=16s$}
    \end{subfigure}
    \hfill
    \begin{subfigure}{0.32\textwidth}
        \centering
        \includegraphics[width=\textwidth]{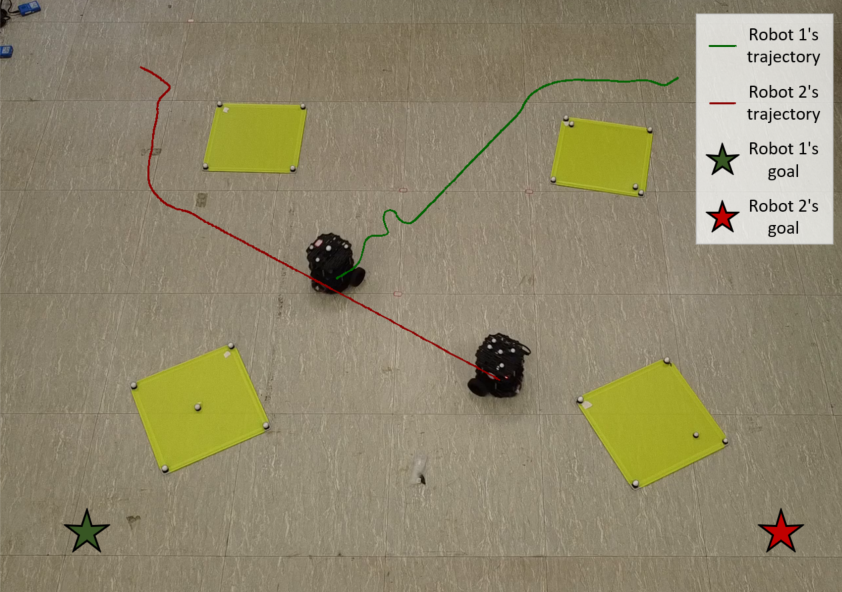}
        \caption{$t=19s$}
    \end{subfigure}
    \hfill
    \begin{subfigure}{0.32\textwidth}
        \centering
        \includegraphics[width=\textwidth]{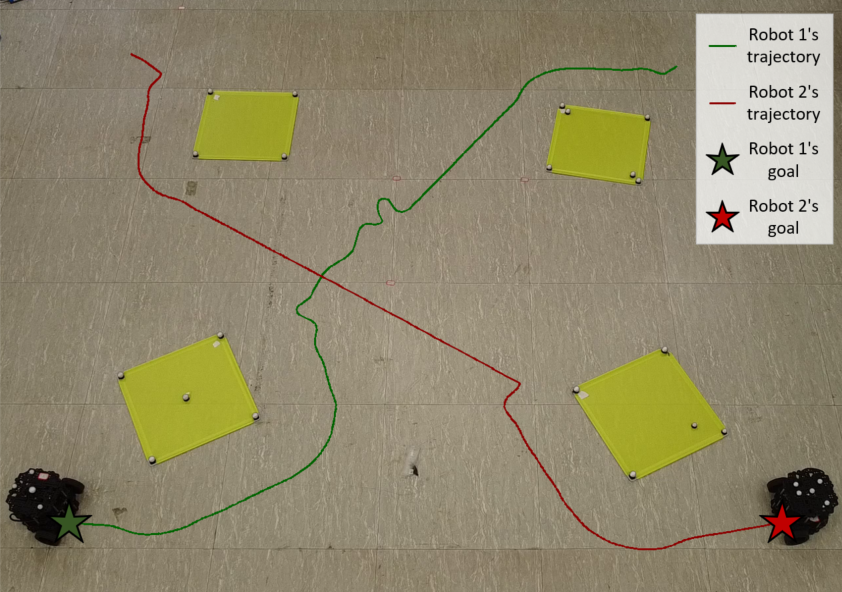}
        \caption{$t=30s$}
    \end{subfigure}
    
    \vspace{0cm}

    \begin{subfigure}{0.144\textwidth}
        \centering
        \adjustbox{raise=1pt}{\includegraphics[width=\textwidth]{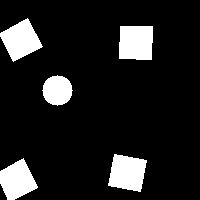}}
        \caption{}
    \end{subfigure}
    \hfill
    \begin{subfigure}{0.175\textwidth}
        \centering
        \includegraphics[width=\textwidth]{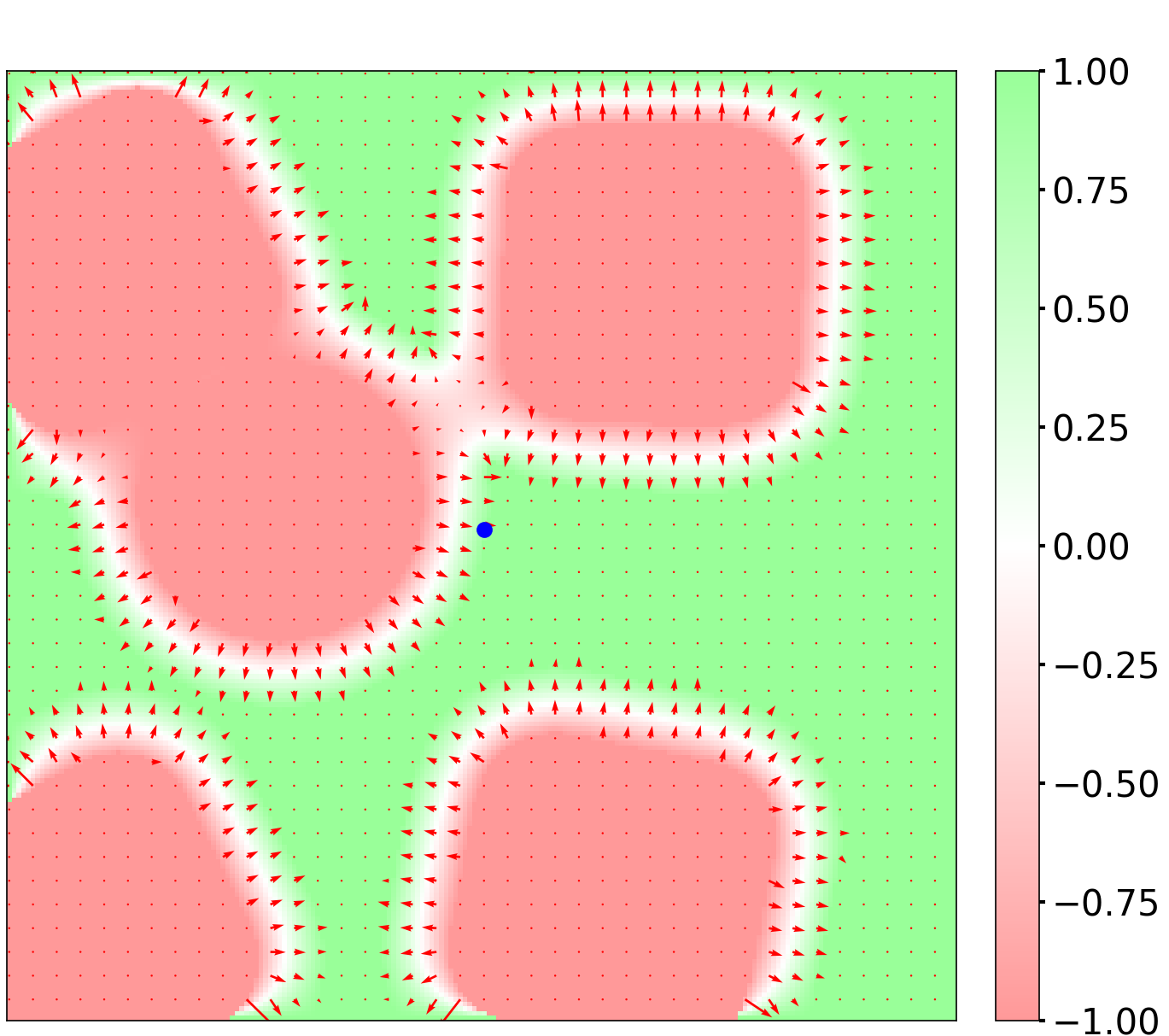}
        \caption{}
    \end{subfigure}
    \hfill
    \begin{subfigure}{0.144\textwidth}
        \centering
        \adjustbox{raise=1pt}{\includegraphics[width=\textwidth]{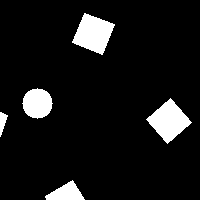}}
        \caption{}
    \end{subfigure}
    \hfill
    \begin{subfigure}{0.175\textwidth}
        \centering
        \includegraphics[width=\textwidth]{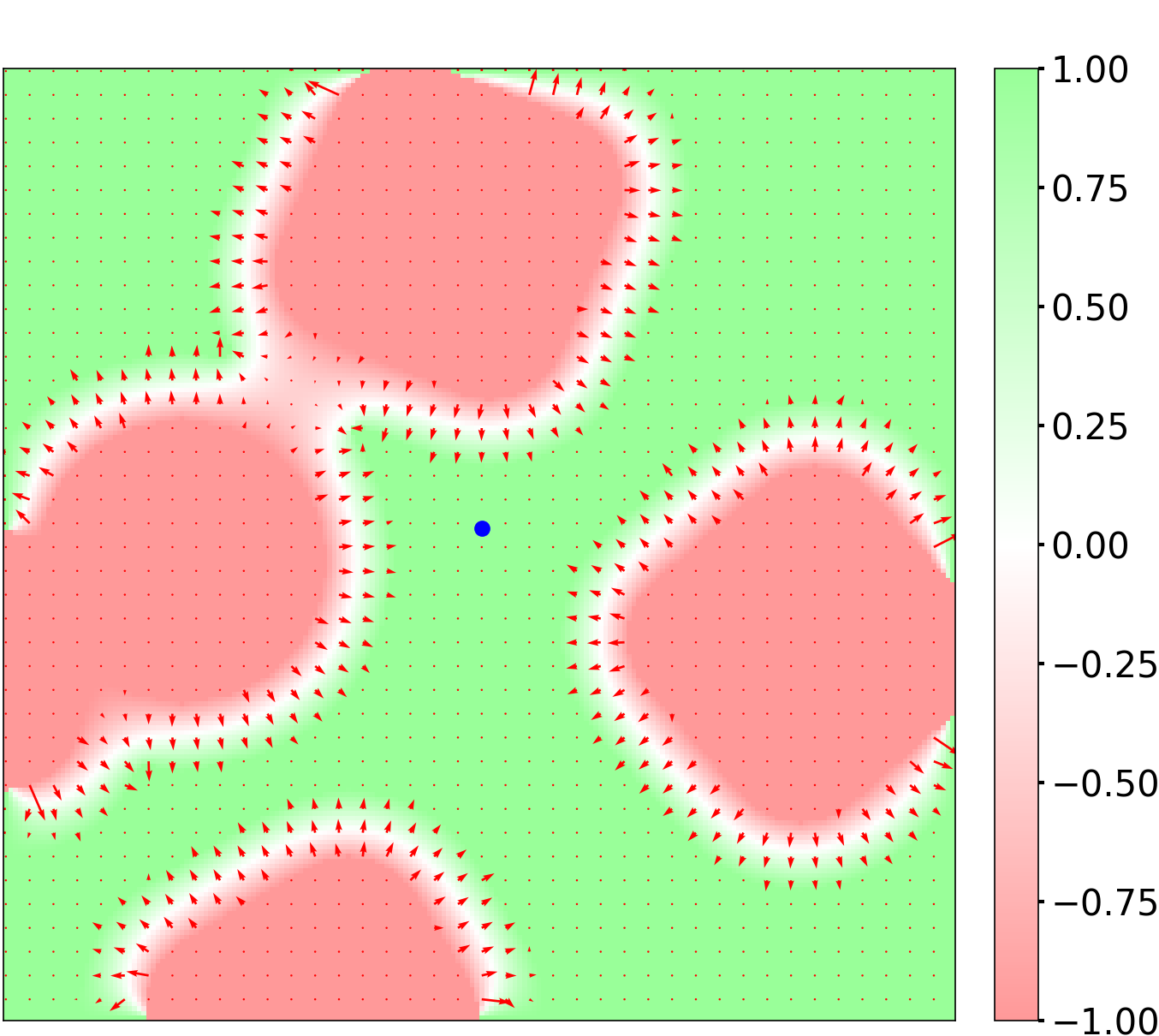}
        \caption{}
    \end{subfigure}
    \hfill
    \begin{subfigure}{0.144\textwidth}
        \centering
        \adjustbox{raise=1pt}{\includegraphics[width=\textwidth]{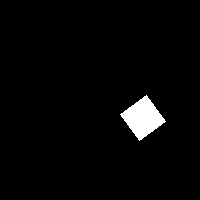}}
        \caption{}
    \end{subfigure}
    \hfill
    \begin{subfigure}{0.175\textwidth}
        \vspace{10pt}
        \centering
        \includegraphics[width=\textwidth]{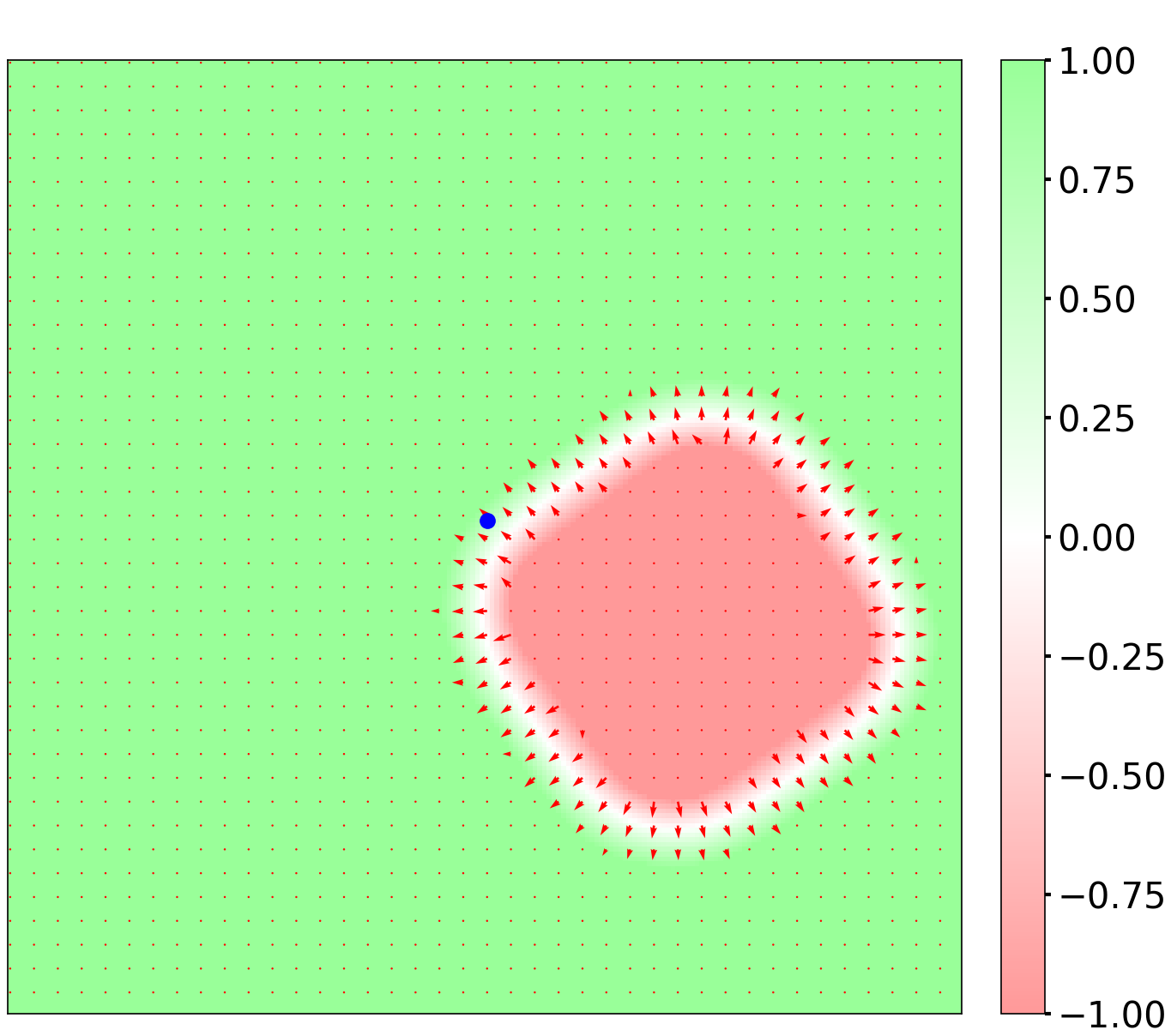}
        \caption{}
    \end{subfigure}

    \vspace{0cm}
    
    \begin{subfigure}{0.144\textwidth}
        \centering
        \adjustbox{raise=1pt}{\includegraphics[width=\textwidth]{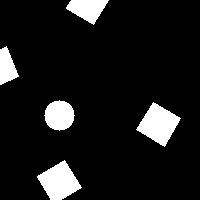}}
        \caption{}
    \end{subfigure}
    \hfill
    \begin{subfigure}{0.175\textwidth}
        \centering
        \includegraphics[width=\textwidth]{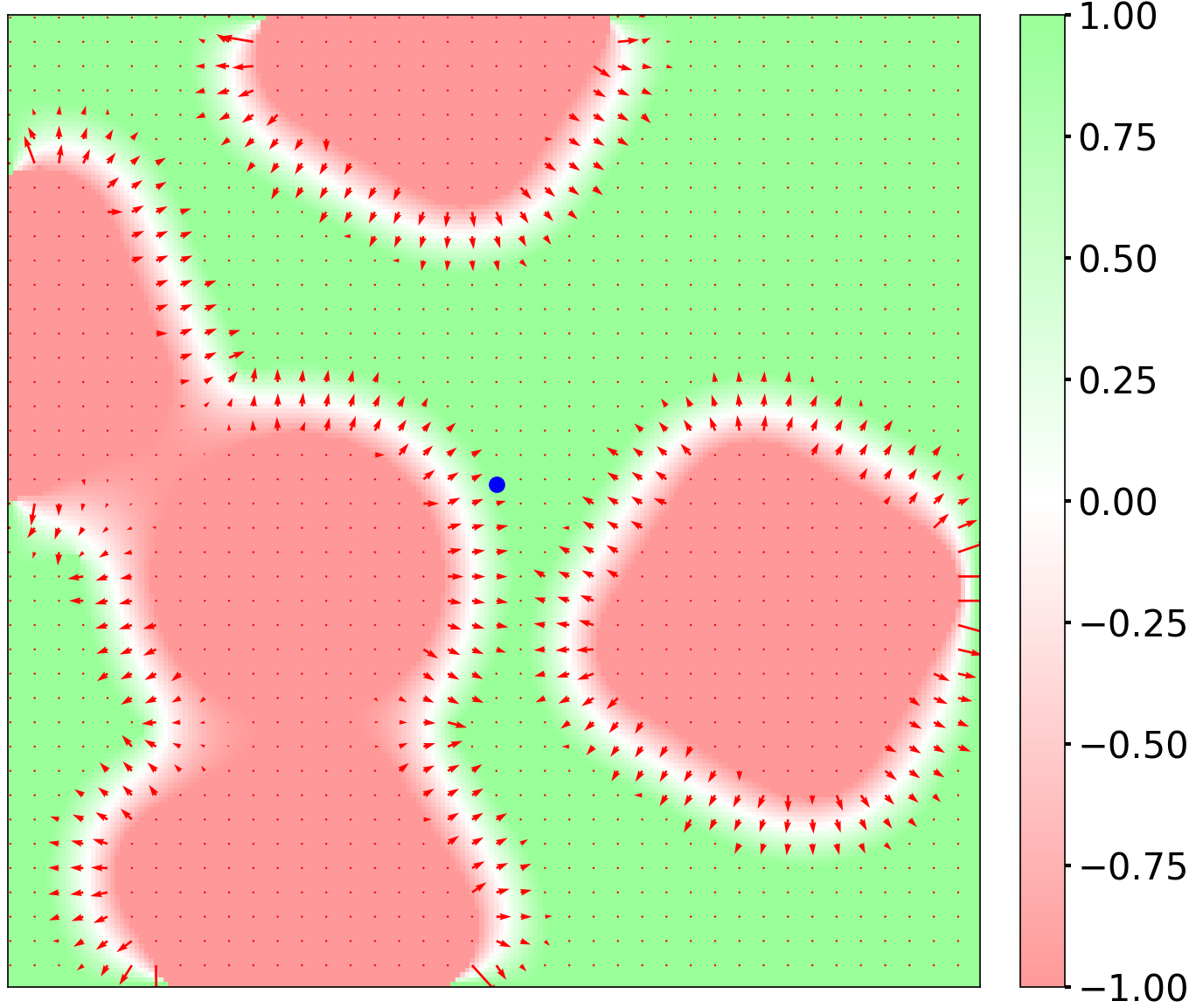}
        \caption{}
    \end{subfigure}
    \hfill
    \begin{subfigure}{0.144\textwidth}
        \centering
        \adjustbox{raise=1pt}{\includegraphics[width=\textwidth]{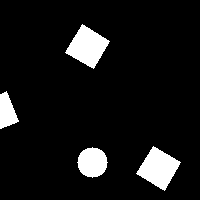}}
        \caption{}
    \end{subfigure}
    \hfill
    \begin{subfigure}{0.175\textwidth}
        \centering
        \includegraphics[width=\textwidth]{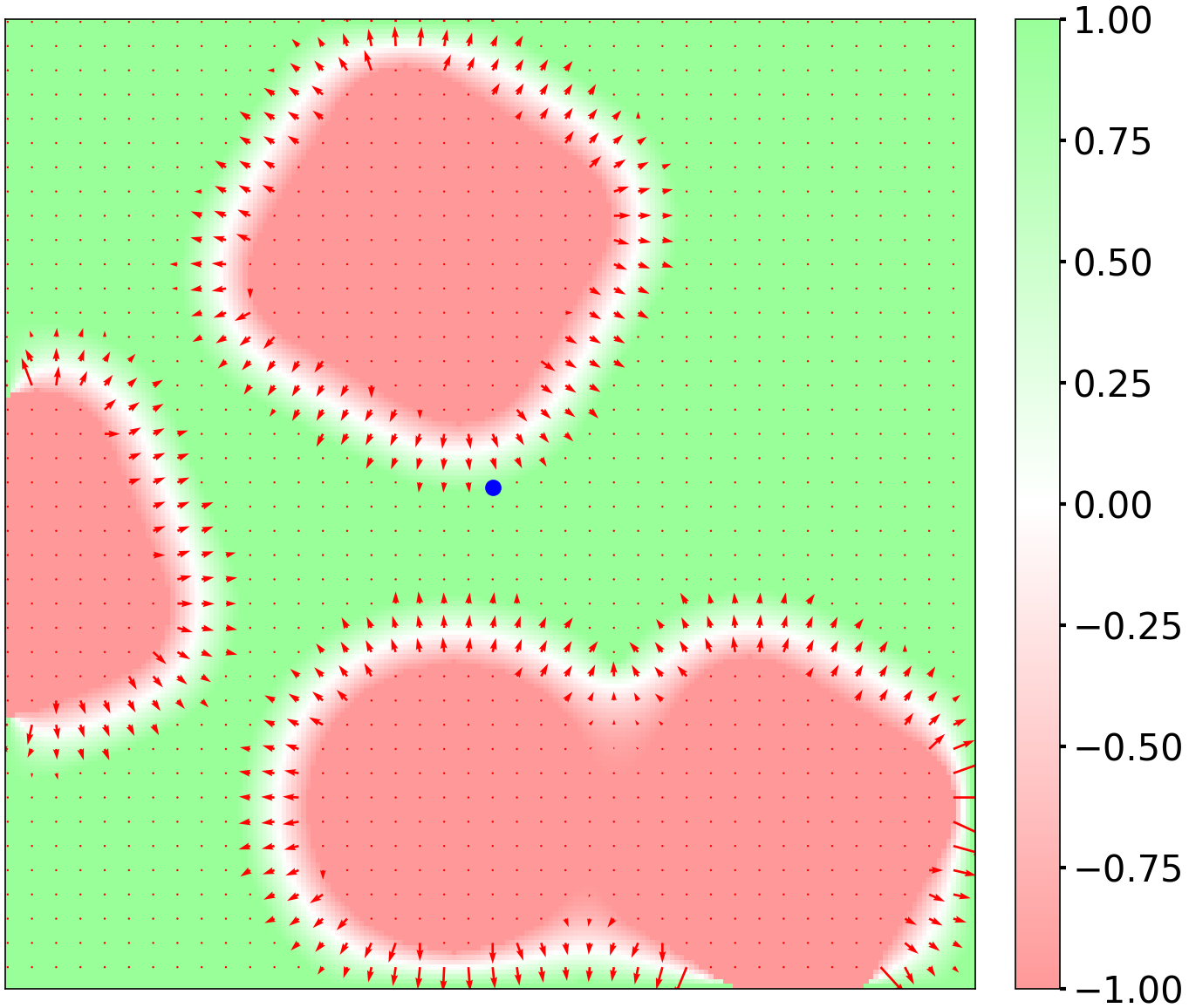}
        \caption{}
    \end{subfigure}
    \hfill
    \begin{subfigure}{0.144\textwidth}
        \centering
        \adjustbox{raise=1pt}{\includegraphics[width=\textwidth]{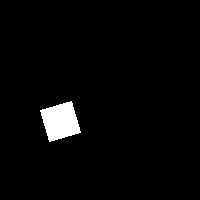}}
        \caption{}
    \end{subfigure}
    \hfill
    \begin{subfigure}{0.175\textwidth}
        \centering
        \includegraphics[width=\textwidth]{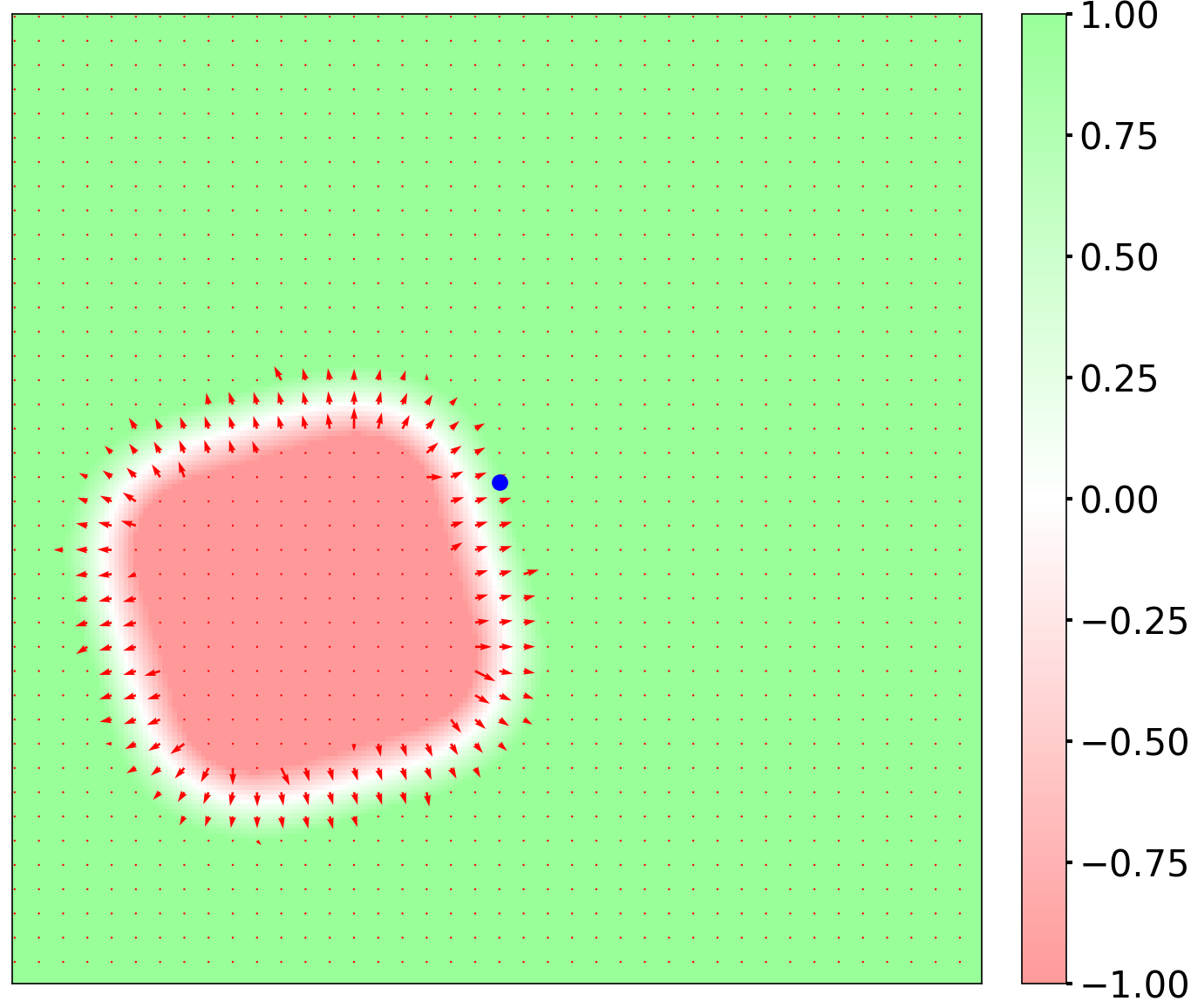}
        \caption{}
    \end{subfigure}
    
    \caption{(a-c) illustrate the trajectories of two robots during navigation. (d-i) present robot 1's local occupancy grid maps (before dilation) along with the corresponding online-generated SSTF-CBFs, synthesized after applying dilation to the OGM by the robot's radius, at time $t=16,19$ and $30$ seconds. Similarly, (j-o) depict the equivalent data for robot 2 at the same time instances. In each CBF color map, red arrows indicate the gradient of the CBF at each position, while the blue dot at the center marks the robot's central position. }
    \label{fig:real}
\end{figure*}
\subsection{Real-World Experiments}
Finally, we conduct the real-world experiment in a $3\text{ m} \times 3\text{ m}$ arena with four yellow $0.35\text{ m}\times 0.35\text{ m}$ static obstacle panels as shown in Figure~\ref{fig:real}. 
Two robots, without any environmental prior knowledge, performed independent navigation tasks toward distinct targets. 

For experimental purpose, we use a {Vicon} motion capture system to provide real-time localization data, enabling the generation of $200 \times 200$ local occupancy grid maps centered on each robot. Each $1 \text{ cm} \times 1 \text{ cm}$ grid cell indicated obstacle presence, with robots treating each other as dynamic obstacles during navigation. Using real-time local OGM, the robots synthesized SSTF-CBF online to ensure safe navigation.

The system utilized ROS as the communication platform, with both robots implementing our algorithm on a single computer that generated and published control commands. The TurtleBots, modeled as circles with a $0.15\text{ m}$ radius, operated at a maximum speed of $0.15m/s$. SSTF-CBF synthesis parameters were set to $a = b = 1$ and $\delta = 0.15$ m, with $\alpha(h)=0.15h$. Figure \ref{fig:real} illustrates the robots' trajectories, demonstrating successful obstacle avoidance and target achievement.

\begin{table}[H]
  \vspace{\fill} 
  \hspace{\fill}   
    \centering
    \caption{Average Data Size and Computation Time}
    \begin{tabular}{cccc}
        \toprule
         & \multirow{2}{*}{Simulation} & \multicolumn{2}{c}{Real-world} \\
        \cmidrule(lr){3-4}
         & & Robot 1 & Robot 2 \\
        \midrule
        Detected obstacle grids & 1627.95 & 2068.70 & 2181.58 \\
        Grids in region  $\mathcal{T}$ to solve& 6962.60 &  6231.87 & 6439.77 \\
        Time for construct matrix(ms) & 4.98 & 3.55 & 3.50 \\
        Time for solving Laplace(ms) & 4.33 & 9.80 & 9.05 \\
        \bottomrule
    \end{tabular}
    \label{tab:time}
\end{table}

\subsection{Results and Discussions}
As detailed in Section \ref{section:CBF-Synthesis}, the online SSTF-CBF synthesis, comprises two key steps executed on different processors for optimal efficiency. First, GPU-based boundary condition determination constructs coefficient matrix $A$ for Laplace equation \eqref{eq:Ax=b} by identifying boundary-adjacent grid cells and their neighboring indices within the thermal field. Second, CPU-implemented subspace iterative methods (GMRES or BiCGSTAB) solve the Laplace equation. Table \ref{tab:time} summarizes the average computational time for CBF synthesis in both simulation and real-world experiments. Obstacle-free scenarios are excluded from the computation time analysis, as the ideal controller can be directly applied without requiring additional processing time. 
These experimental results demonstrate the  computational efficiency of our methods, with an average SSTF-CBF synthesis time of $9.31$ ms ($4.98$ ms + $4.33$ ms) in simulations.



\section{Conclusion}
\label{section:conclusion}
This paper proposes a new approach for online synthesis of CBFs based on local occupancy grid maps. The main features of the proposed approach are twofold. First, it employs a single CBF constraint to ensure safety across varying obstacle numbers and shapes, simplifying the QP optimization problem and reducing conservatism. Second, inspired by steady-state thermal fields and leveraging the sparsity of the coefficient matrix in Laplace's equation, the safety value can be computed efficiently in real time, even for high-dimensional OGMs inputs. The method's efficacy is demonstrated through simulations and real-world experiments, achieving millisecond-level synthesis times on \(200 \times 200\) local OGMs. Future work will focus on integrating search techniques for multi-step motion planning and extending the approach to multi-robot safety control.





\balance
\bibliographystyle{IEEEtranBST/IEEEtran}
\bibliography{ieeeconf/reference}

\end{document}